\documentclass{article} 
\usepackage{nips15submit_e,times}
\usepackage{hyperref}
\usepackage{url}
\nipsfinalcopy

\usepackage{url}
\usepackage{amsmath}
\usepackage{amssymb}
\usepackage{amsthm}
\usepackage{amsfonts}
\usepackage{braket}
\usepackage{complexity}
\usepackage{graphicx}
\usepackage{dcolumn}
\usepackage{bm}
\usepackage{hyperref}
\usepackage{enumerate}
\usepackage{algorithm}
\usepackage{algpseudocode}

\usepackage{multirow}

\usepackage{xr}
\newtheorem{theorem}{Theorem}
\newtheorem{lemma}{Lemma}

\newtheorem{corollary}{Corollary}

\newenvironment{proofof}[1]{\begin{trivlist}\item[]{\flushleft\it
Proof of~#1.}}
{\qed\end{trivlist}}

\newcommand{\eq}[1]{\hyperref[eq:#1]{(\ref*{eq:#1})}}
\renewcommand{\sec}[1]{\hyperref[sec:#1]{Section~\ref*{sec:#1}}}
\newcommand{\app}[1]{\hyperref[app:#1]{Appendix~\ref*{app:#1}}}
\newcommand{\fig}[1]{\hyperref[fig:#1]{Figure~\ref*{fig:#1}}}
\newcommand{\thm}[1]{\hyperref[thm:#1]{Theorem~\ref*{thm:#1}}}
\newcommand{\lem}[1]{\hyperref[lem:#1]{Lemma~\ref*{lem:#1}}}
\newcommand{\tab}[1]{\hyperref[tab:#1]{Table~\ref*{tab:#1}}}
\newcommand{\cor}[1]{\hyperref[cor:#1]{Corollary~\ref*{cor:#1}}}
\newcommand{\alg}[1]{\hyperref[alg:#1]{Algorithm~\ref*{alg:#1}}}
\newcommand{\defn}[1]{\hyperref[def:#1]{Definition~\ref*{def:#1}}}

\newcommand{\expect}[2]{{\mathbb{E}_{#2}}\!\left\{#1 \right\}}

\newcommand{\CRej}{\text{rejection filtering}}
\newcommand{\CSMC}{\text{SMC}}

\newcommand{\NN}{\mathrm{N}}

\newcommand{\tout}[1]{{}}

\begin{document} 

\title{Approximate Bayesian Inference via Rejection Filtering}
\author{
Nathan Wiebe \\
Microsoft Research\\
\And
Christopher Granade \\
 University of Sydney\\
\AND
Ashish Kapoor \\
Microsoft Research \\
\And
Krysta M. Svore \\
Microsoft Research \\
}

\maketitle

\begin{abstract}
We introduce a method, rejection filtering, for approximating Bayesian inference
using rejection sampling.
We not only make the process efficient, but also dramatically
reduce the memory required relative to conventional methods
by combining rejection sampling with particle filtering to estimate the first two moments of the posterior distribution.
We also provide an approximate form of rejection sampling that makes rejection filtering tractable
in cases where exact rejection sampling is not efficient.
Finally, we present several numerical examples of rejection filtering that show its ability to track
time dependent parameters in online settings, and show its performance on MNIST classification problems.
\end{abstract}

\section{Introduction}
\label{sec:intro}

Particle filters have become an indispensable tool for model selection, object tracking
and statistical inference in
high--dimensional problems~\cite{doucet2000sequential,del2012adaptive,van2000unscented,liu2001combined}. 
While particle filtering works well in many conventional settings, 
the method is less well-suited when the
user faces severe memory restrictions.  

Memory restricted problems are more than just curiosities. In control problems
in electrical engineering and experimental physics, for instance,
it is common that the dynamics of
a system can radically change over the time required to communicate between a system
and the computer used to control its dynamics \cite{halloin_long_2013,shulman_suppressing_2014}.  This latency can be reduced to acceptable levels by allowing the inference to performed inside
the device itself \cite{lavalle_sensor_2013}, but this often places prohibitive restrictions on the processing power and memory of the embedded processors.
These restrictions can preclude the use of traditional particle filter methods.

We present an approach that we call \emph{rejection filtering} (RF) that efficiently samples from
an approximation to the posterior by using rejection sampling and resampling together.
This allows rejection filtering to estimate the first two moments of the posterior distribution while storing no more
than a constant number of samples at a time in typical use cases.
Rejection filtering therefore can require significantly
less memory than traditional particle filter methods. Moreover, RF can be
easily parallelized at a fine-grained level, such that it can be used with an
array of small processing cores. Thus, \CRej~is well suited for inference
 using hybrid computing and memory-restricted platforms.
For example, \CRej~allows for inference to be embedded in novel contexts such
as very small cryogenic controllers \cite{hornibrook_cryogenic_2015}, or for
integration with existing memory-intensive digital signal processing systems
\cite{casagrande_design_2014}.

We also show that the advantages of rejection filtering are retained in the \emph{active learning} case as well,
wherein samples correspond to experiments that can be selected and used to optimize the inference procedure.
In particular, our approach uses well-motivated experiment design heuristics in conjunction with rejection
filtering.
This amounts to a form of selective sampling that is computationally inexpensive, easy to parallelize, simple to program and can operate in a much more memory restricted environment than existing approaches~\cite{sivaraman2010general,kapoor2007active}.

\section{Rejection Filtering}
\label{sec:method}


We start by discussing rejection sampling methods, as we
will draw ideas from these methods to formulate our rejection filter algorithm.
For simplicity we take all variables to be discrete in the following, however, the generalization to the continuous
case is trivial.
Rejection sampling provides a simple and elegant approach to Bayesian inference. 
It samples
from the posterior distribution by first sampling from the prior $P(x)$.
The samples are then accepted with
probability $P(E|x)$.  The probability 
distribution of the accepted samples is 
\begin{equation}
  \frac{P(E|x)P(x)}{\sum_x P(E|x)P(x)}= P(x|E),
\end{equation}
and the probability of drawing a sample which is then accepted is $\sum_x P(E|x)P(x)=P(E)$.  

This probability
can be increased by instead accepting a sample with probability $P(x|E)/\kappa_E$ where
$\kappa_E$ is a constant that depends on the evidence $E$ such that $P(x|E) \le \kappa_E \le 1$.
Rescaling the likelihood does not change the posterior probability.
It does however make the probability of acceptance $P(E)/\kappa_E$, which
can dramatically improve the performance when rare events are observed.

Two major drawbacks to this approach have prevented its widescale adoption.  The first is that the probability of successfully updating shrinks exponentially with the number of updates
in online inference problems.
The second is that the constant $\kappa_E$ may not be precisely known such that and thus $P(E|x)$ cannot be appropriately rescaled to avoid exponentially small likelihoods.  Given that the dimension of the Hilbert space that the training vectors reside in scales exponentially with the number of features, exponentially small probabilities are the norm rather than the exception. 
Our aim is to address these problems
by using approximate, rather than exact, rejection sampling, and
by combining rejection sampling with ideas from  \emph{sequential Monte Carlo} methods~\cite{doucet2000sequential,del2012adaptive,van2000unscented,liu2001combined} to allow approximate inference to be performed while only storing summary statistics of the prior distribution.


We make rejection sampling efficient by combining it with particle
filtering methods through \emph{resampling}.
Rejection filtering does not try to
to propagate samples through many rounds of rejection sampling, but instead
uses these samples to inform a new model for the posterior distribution. 
For example, if
we assume that our prior distribution is a Gaussian, then a Gaussian model for the posterior
distribution can be found by computing the mean and the covariance matrix for the samples
that are accepted by the rejection sampling algorithm.  This approach is
reminiscent of assumed density filtering~\cite{minka_expectation_2001}, which uses an analogous strategy
for modeling the prior distribution but is less memory efficient than
our method.

\begin{algorithm}[t!]
    \caption{Update for \CRej}
    \label{alg:crej}
    \begin{algorithmic}
        \Require Array of evidence $\vec{E}$, number of attempts $m$, a constant $0<\kappa_E\le 1$, a recovery factor $r \ge 0$ and the prior $P$.
        \Function{RFUpdate}{$\vec{E}$, $\mu$, $\Sigma$, $m$, $\kappa_E$, $r$}
    \State{$(M,S,N_a) \gets 0$}
          \For{$i \in 1 \to m$}
            \State $x \sim P$
            \State $u \sim \operatorname{Uniform}(0, 1)$
            \If{$\prod_{E\in \vec{E}}\min\left(P(E | x)/\kappa_E,1\right) \ge u$} 
            \State $M \gets M+ x$
            \State $S \gets S+ xx^T$
    \State $N_a \gets N_a +1$.
            \EndIf
    \EndFor
    \If{$N_a \ge 1$}
       \State $\mu\gets M/N_a $
       \State $\Sigma \gets \frac{1}{N_a -1}\left(S - N_a \mu\mu^T \right)$
    \State\Return $(\mu,\Sigma,N_a)$
   \Else
    \State\Return $(\mu, (1+r)\Sigma),N_a)$

   \EndIf
          
        \EndFunction
    \end{algorithmic}
\end{algorithm}

Our method is described in detail in~\alg{crej}. We discuss the efficiency of the algorithm in the following theorem.
We consider an algorithm to be efficient if it runs in $O(\operatorname{poly}(\operatorname{dim}(x)))$ time, where ${\rm dim}(x)$ is the number of features in the training vectors.

\begin{theorem}
Assume that $P(E|x)\le \kappa_E$ can be computed efficiently for all hypotheses $x$, $\sum_x P(E|x)/\kappa_E$ is at most polynomially small for all
evidences $E$, $P(x)$ can be efficiently sampled and an efficient sampling algorithm for $\operatorname{Uniform}(0,1)$ is provided.  \alg{crej} 
can then efficiently compute the mean and covariance of $P(x|E)$ within error $\epsilon$ in the max--norm using $O({\rm dim}(x)^2\log({\rm dim}(x)/\epsilon))$ memory.\label{thm:crej}
\end{theorem}
A formal proof is given in \app{proofs}.  The intuition is that a sample can be non--deterministically drawn from
the posterior distribution by drawing a
sample from the prior distribution and rejecting it with probability $P(E|x)$.  Incremental formulas are
used in~\alg{crej} to estimate the mean and the covariance using such samples, which obviates the need
to store $O(1/\epsilon^2)$ samples in memory in order to estimate the moments of the posterior distribution within error $\epsilon$.
In practice, one can use the Welford algorithm \cite{welford_note_1962} to  accumulate means
and variances more precisely, but doing so does not change the asymptotic scaling with $\epsilon$ of the memory required by~\alg{crej}.

Width can be traded for depth by batching the data and processing each of
these pieces of evidence separately. 
We use a computational model wherein $N_{\rm batch}$ processing
nodes send a stream of the incremental means and
covariance sums to a server that combines them to produce the
model used in the next step of the inference procedure.
Pseudocode for this version  is given in \app{batched-updates}.

\subsection{Bayesian Inference using Approximate Rejection Sampling}
\label{sec:approx-rej}

Having used resampling to unify rejection sampling and particle filtering,
we can significantly improve the complexity of the resulting rejection filtering
algorithm by relaxing from exact rejection sampling.  Approximate rejection
sampling is similar to rejection sampling except
that it does not require that $P(E|x) \le \kappa_E$.  This means that the rescaled
likelihood $P(E|x)/\kappa_E$ is greater than $1$ for some
configurations.  This inevitably results in errors in the posterior distribution but can make the inference process much more efficient
in cases where a tight bound is unknown or when the prior has little support over the region where $P(E|x)/\kappa_E >1$.

The main question remaining is how  the
errors incurred from $P(E|x) > \kappa_E$ impact the posterior distribution.
To understand this, let us define
\begin{equation}
{\rm bad} := \left\{x: {P(E|x)} >{\kappa_E}\right\}.
\end{equation}
If the set of bad configurations is non--empty then it naturally leads to errors in the posterior and
can degrade the success probability in rejection filtering.  Bounds on these effects are
provided below.

\begin{corollary}\label{cor:badalgorithm}
If the assumptions of~\thm{crej} are met, except for the requirement that $P(x|E) \le \kappa_E$, and
$$\sum_{x\in {\rm bad}}  \left([P(E|x)-\kappa_E] P(x)\right) \le \delta P(E),$$
  then approximate rejection sampling is  efficient and samples from a distribution $\rho(x|E)$ such that ${\sum_x \sqrt{\rho(x|E) P(x|E)}} \ge 1-\delta$.
The probability of accepting a sample is at least ${P(E) (1-\delta)}/{\kappa_E}$.\label{thm:kappa}
\end{corollary}
\begin{proof}
Result follows directly from~\thm{crej} and Theorem 1 in~\cite{WKGS15}.
\end{proof}

\cor{badalgorithm} shows that taking a value of $\kappa_E$ that is too small for $P(E|x)/\kappa_E$ to be a valid likelihood function does not necessarily result in substantial errors in the posterior distribution.
We further elaborate in \app{sensitivity-kappa} with a numerical example.
This leads to an efficient method for approximate sampling from the posterior distribution assuming that $\delta$ is constant and $P(E|x)/\kappa_E$ is at most polynomially small.  Furthermore, it remains incredibly space efficient since the posterior distribution does not have to be explicitly stored to use rejection filtering.

\section{Numerical Experiments}
\begin{figure}
\begin{minipage}[b]{0.45\linewidth}
    \center{
        \includegraphics[width=\columnwidth]{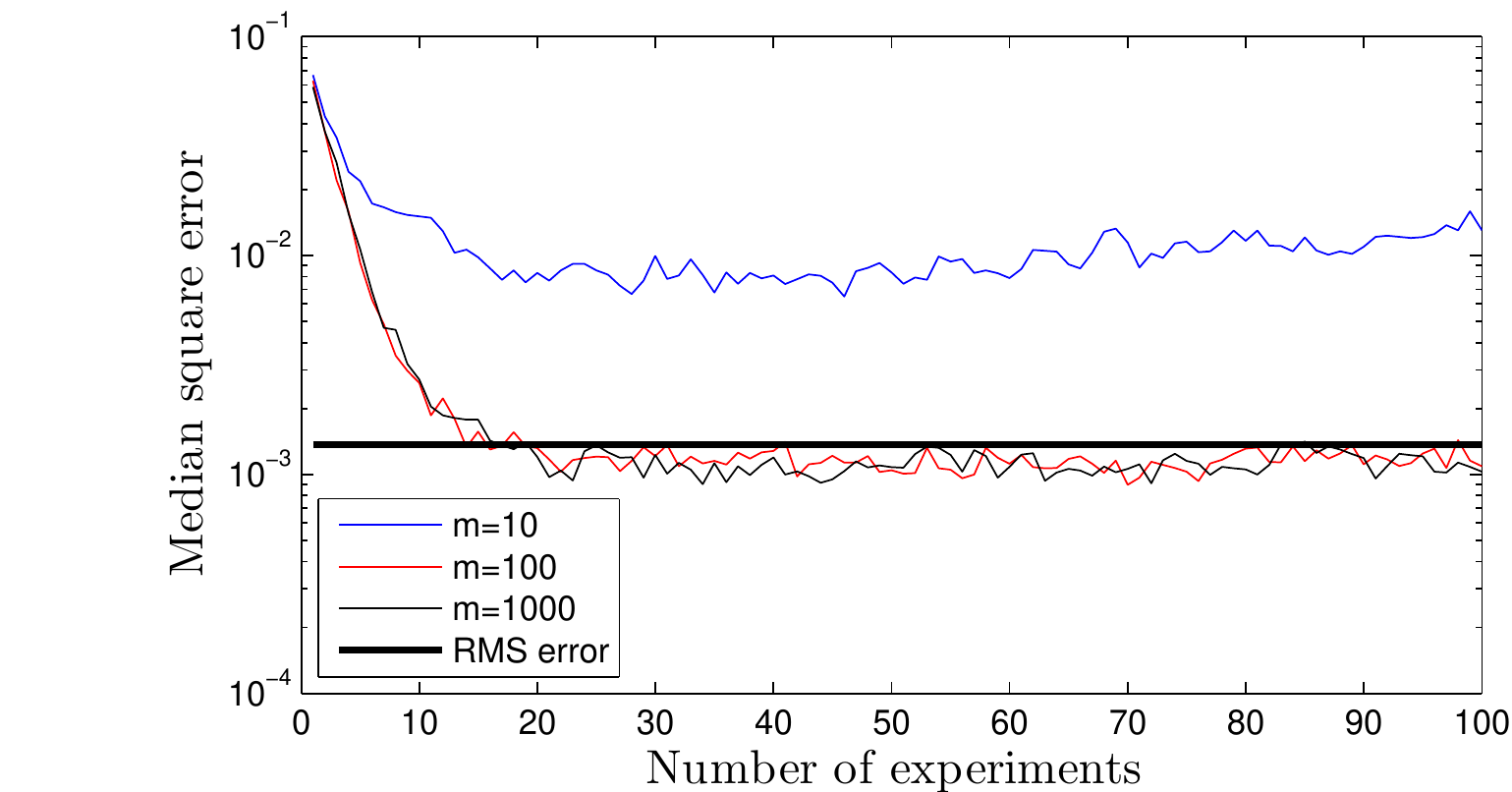}
    }
    \caption{
        \label{fig:crej-diffusion}
        Rejection filtering for the frequency estimation problem.  As expected, the error asymptotes to approximately the root mean square error.
    }
\end{minipage}
\hspace{1mm}
\begin{minipage}[b]{0.45\linewidth}
\centering
\includegraphics[width=0.45\columnwidth]{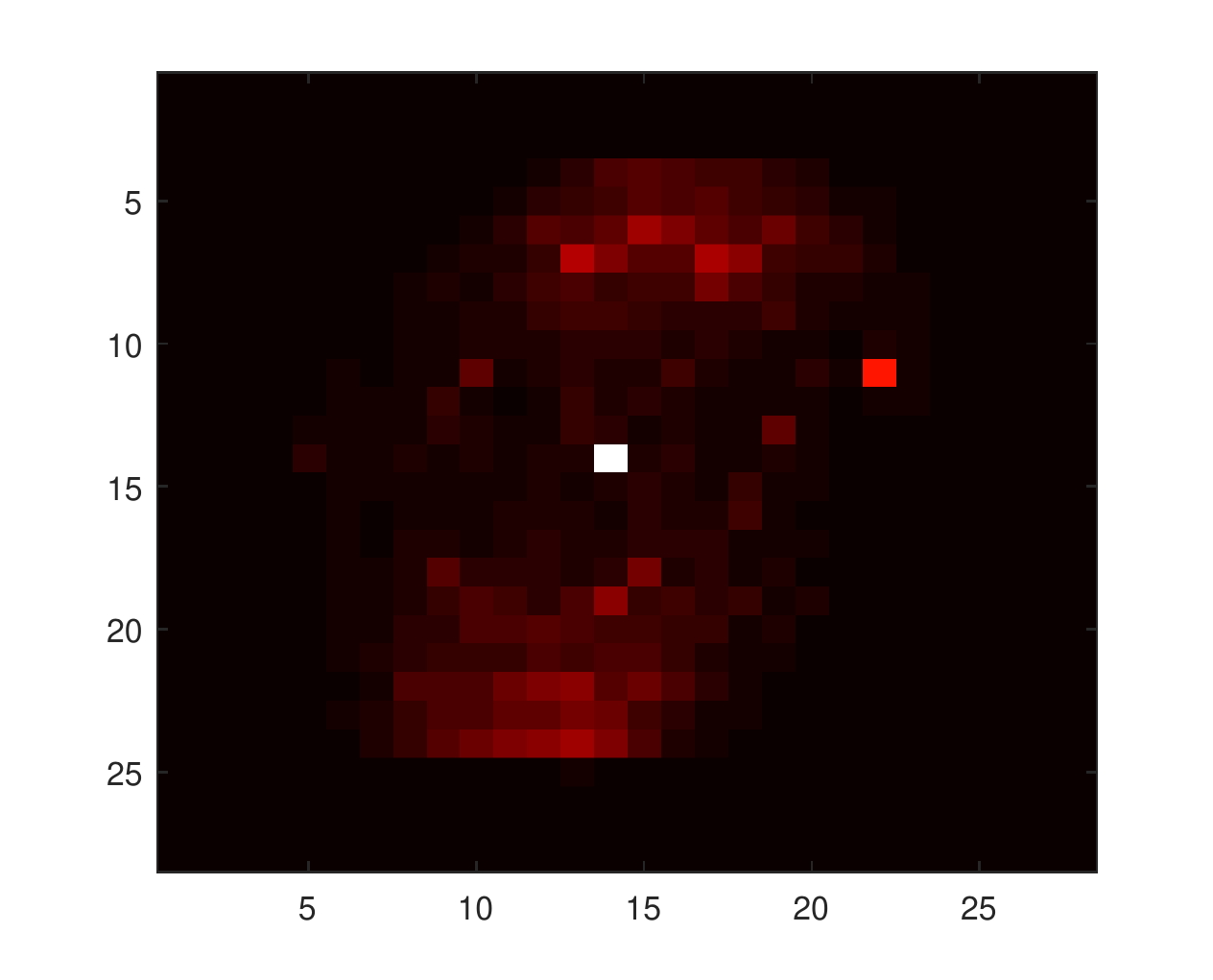}
\includegraphics[width=0.45\columnwidth]{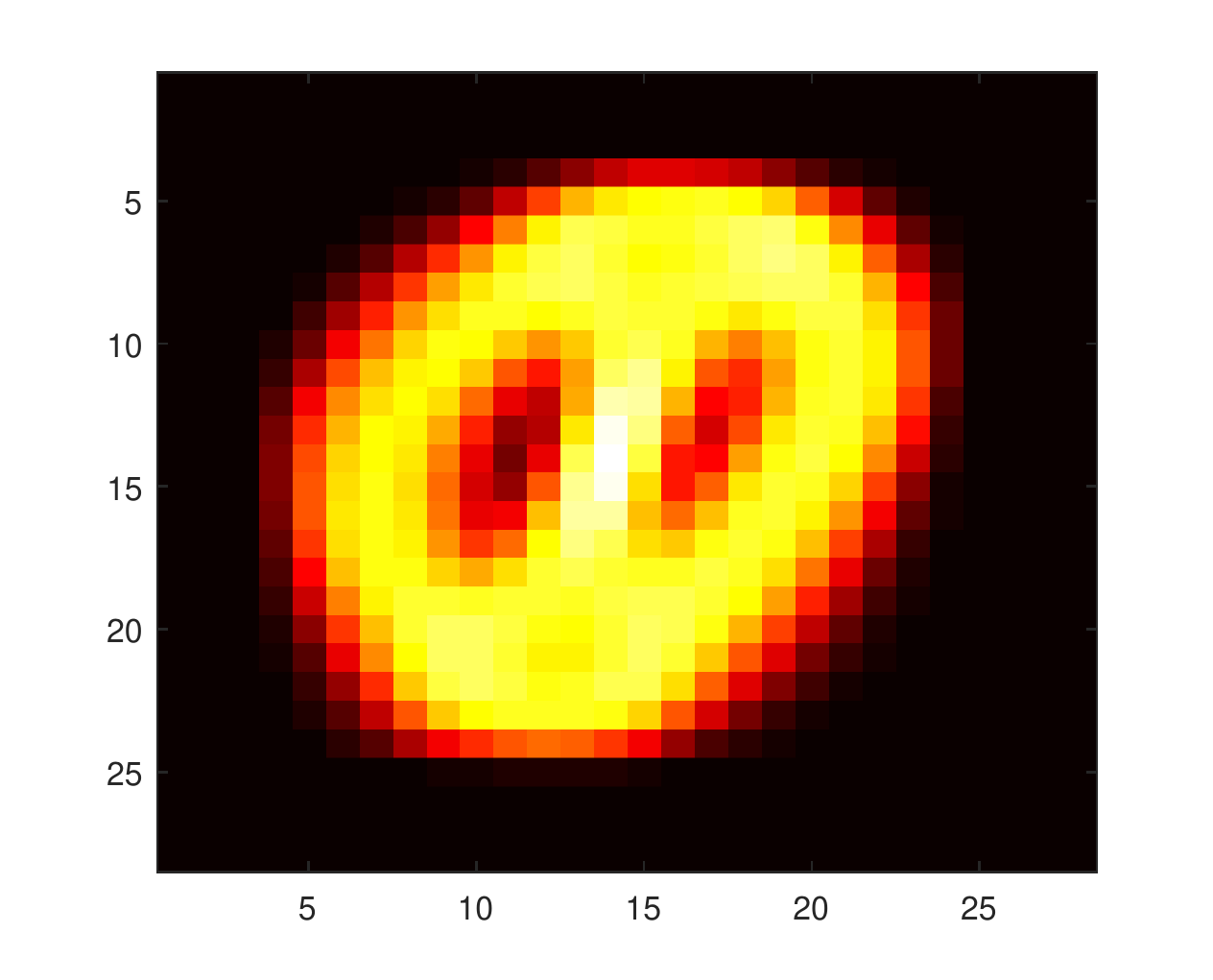}\\
\includegraphics[width=0.45\columnwidth]{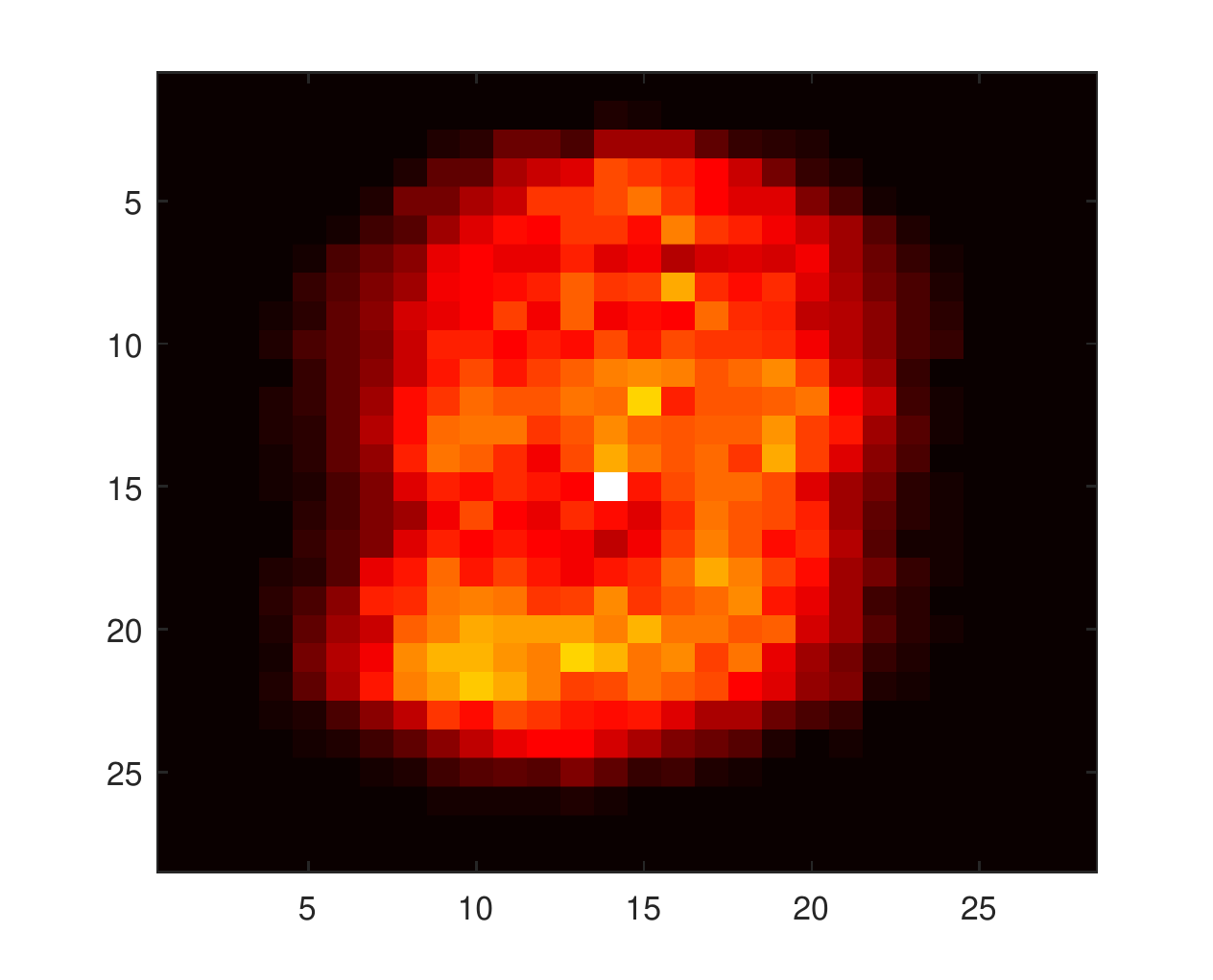}
\includegraphics[width=0.45\columnwidth]{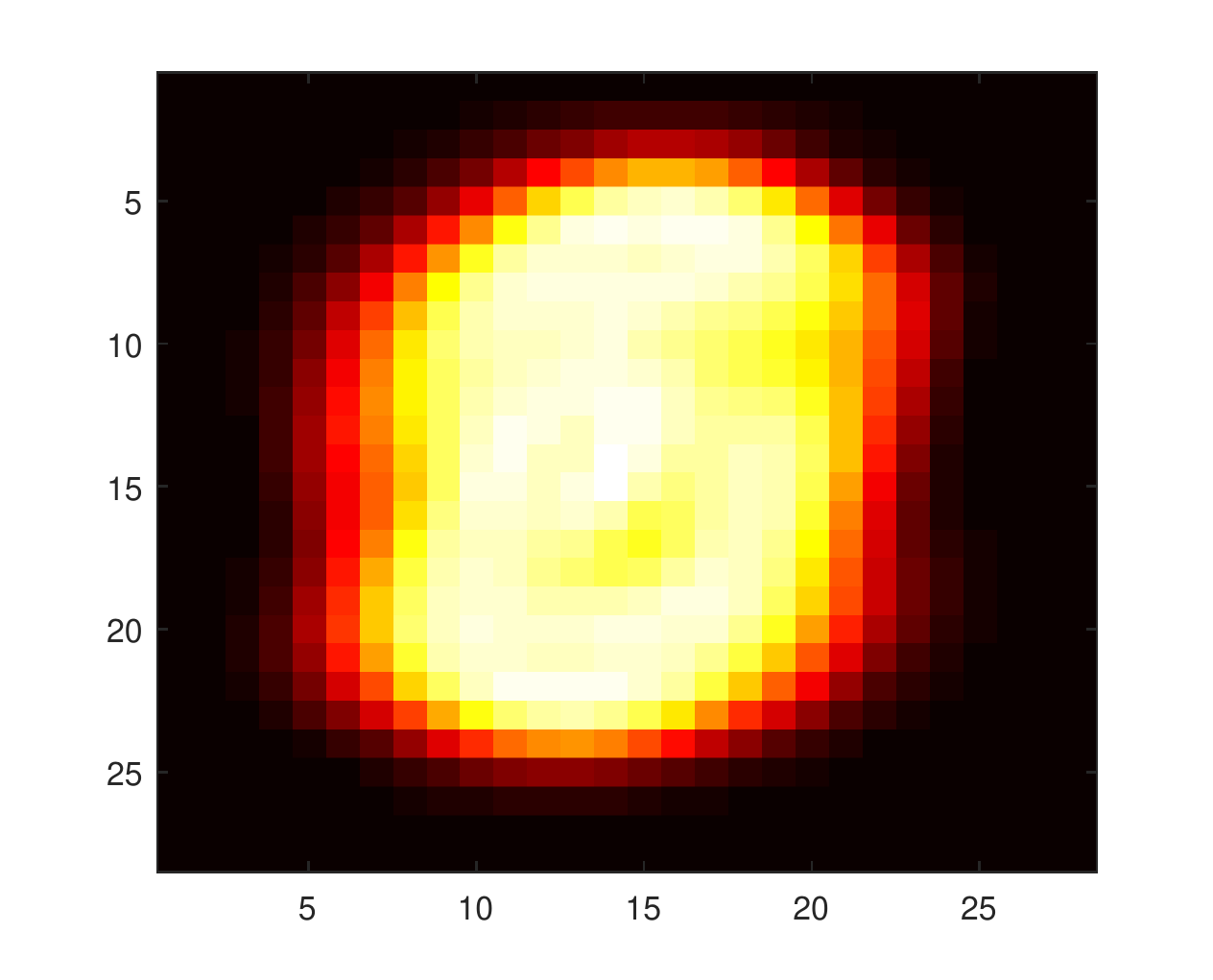}
\caption{(left) Heat map of frequency pixel is queried using RF.  (right) Variance of the pixels over data set.
Bottom plots are for even vs odd digits top is zero vs one.}\label{fig:HM}
\end{minipage}
\end{figure}

Our first numerical experiment is a proof of principle experiment involving inferring an unknown, time--dependent, frequency of an oscillator.
This application is of central importance in quantum mechanics wherein the precise estimation of such frequencies involves substantial
computational and experimental effort using traditional techniques~\cite{wiebe_efficient_2015,ferrie_how_2013}.  In this case, measurements have two outcomes and the likelihood function takes the form
$
P(0|\phi(t_k);M,\theta) = \cos^2(M[\phi(t_k)-\theta]),
$
where $\phi(t_k)$ is the unknown frequency evaluated at time $t_k$ and $M$ and $\theta$ are experimental settings that we optimize over to find a near--optimal
experiment given the current knowledge of the system.  We take $\phi(t_k)$ to follow a normal random walk in $k$ with standard deviation $\pi / 120$.

\fig{crej-diffusion} shows that rejection filtering is capable of accurately tracking this unknown frequancy using fewer than $100$ samples drawn from the prior distribution.  Furthermore, rejection filtering is capable of achieving this using less than $1$ kilobit of memory.  In contrast, sequential Monte--Carlo methods using Liu--West resampling  typically require roughly $100$ times the memory that rejection filtering requires because of their need to store an entire particle cloud rather than a single sample.

We provide further details of our experimental design heuristics and apply RF to digit classification in an active learning setting in the appendix.  We see,  in~\fig{HM}, that some low variance features are frequently used to classify digits whereas many high variance features are not.  We show in the appendix that RF can also beat KNN  for an appropriate choice of resampler.

\section{Conclusion}
\label{sec:conclusions}

We introduce a method, rejection filtering, that combines
rejection sampling with particle filtering. Rejection filtering retains many of the benefits
of each, while using substantially less memory than conventional methods for Bayesian inference in typical use cases.
In particular, if a Gaussian resampling algorithm is used then our method only requires remembering a single sample at a time, making it ideal
for memory-constrained and active learning applications.
We further illustrate the viability of our rejection filtering approach through numerical experiments
involving tracking the time-dependent drift of an unknown frequency and also in handwriting recognition.

While our work has shown that rejection sampling can be a viable method for performing Bayesian inference, there are many
avenues for future work that remain open.  One such avenue involves investigating whether ideas borrowed from the particle
filter literature, such as the unscented transformation~\cite{van2000unscented} or genetic mutation-selection algorithms~\cite{del2012adaptive,del2000branching}, can be adapted to fit our setting.  An even more exciting
application of these ideas may be to examine their application in online data acquisition in science and engineering.  Rejection
filtering provides the ability to perform adaptive experiments using embedded hardware, which may lead to a host of applications
within object tracking, robotics, and experimental physics that are impractical with existing technology.


\appendix

\section{Supplemental Analysis}

\subsection{Filtering Distributions for Time-Dependent Models}
\label{sec:time-dep}

Bayesian inference has been combined with
time-dependent models 
to perform object tracking and acquisition in many particle filter applications~\cite{isard_condensationconditional_1998,gustafsson_particle_2002}. Here, we show that
\CRej~naturally encompasses these applications by convolving posterior
distributions with Gaussian update kernels.

In time-dependent applications the true model is not
stationary, but rather changes as observations are made.  This poses a
challenge for na\i ve applications of Bayesian inference because drift in the
true model can cause it to move outside of the support of the  prior
distribution.  This drift results in the online inference algorithm failing to track
an object that moves suddenly and unexpectedly.

To see how this can occur in problems where the true parameters are time-dependent, consider the following likelihood function for a Bernoulli experiment
and a family of prior distributions with mean $\bar{x}$ and variance $\sigma$ such that
the overlap between the likelihood and the prior is given by
\begin{equation}
    \sum_x P(0|x; \bar{x}, \sigma(x)) P(x) \le e^{-|x_{\rm true} - \bar{x}| \gamma/\sigma}.
\end{equation}
If $\sigma$ is small then the small deviations of $x_{\rm true}$ away from $\bar{x}$ introduced by neglecting the time-dependence of $x_{\rm true}$ can cause the inner product to become exponentially small.
This in turn causes the complexity of resampling to be exponentially large, thereby removing guarantees of efficient learning.

Such failures in \CRej~are heralded by tracking the total number
of accepted particles $N_a$ in each update.  This is because $N_a$ estimates
$P(E) = \sum_x P(E | x) P(x)$.
Alternatively, we can do better by
instead incorporating a prediction step that diffuses the model parameters
of each particle \cite{isard_condensationconditional_1998}.
In particular, by convolving the prior with a filter function such as a
Gaussian, the width of the resultant distribution can be increased without
affecting the prior mean. 
In a similar way, \CRej~can be extended to include diffusion by using a resampling kernel
that has a broader variance than that of the accepted posterior samples. Doing so
allows \CRej~to track stochastic processes in a similar way to \CSMC, as described
 in \sec{numerical-experiments}.

Formally, we model our posterior distribution as
\begin{equation}
  P(x|E;t_{k+1}) = P(x|E;t_k) \star \mathcal{B}(0,\eta(t_{k+1} - t_k)),
\end{equation}
where $\mathcal{B}$ is a distribution with zero mean and variance
$\eta$ and $\star$ denotes a convolution over $x$.
Convolution is in general an expensive operation, but for cases where \CRej~uses a
Gaussian model for the posterior distribution, the resulting distribution
remains Gaussian under the convolution if $\mathcal{B}$ is also a Gaussian.
If the variance of the prior distribution is $s$ then it is easy to see from
the properties of the Fourier transform that the variance of
$\tilde{P}(x|E;t)$ is $s+\eta (t_{k+1} - t_k)$ and the mean remains $\bar{x}$.

\subsection{Model Selection}
\label{sec:model-sel}

The ability of \CRej~to include time-dependence
is especially useful when combined with Bayesian model selection.
Since the random variates of $N_a$ drawn at each step give a
frequency drawn from the \emph{total likelihood} $P(E) = \expect{x}{P(E |
x)}$, we can use \CRej~to estimate Bayes factors between two
different likelihood functions. In particular, the probability that a hypothesis $x$ will be accepted
by \alg{crej} is $P(E | x)$, so that marginalizing gives the desired $P(E)$.
Thus, $N_a$ at each step is drawn from a binomial distribution with mean $m P(E)$.
Using hedged maximum likelihood estimation \cite{ferrie_estimating_2012}, we can then
estimate $P(E)$ given $N_a$, even in the cases that $N_a = 0$ or $m$.

Concretely, consider running \CRej~in parallel for two
distinct models $M \in \{A, B\}$, such that all likelihoods are conditioned on a value
for $M$, $P(E | x, M)$. The estimated total likelihoods for each \CRej~run then give an
estimate of the Bayes factor $K$ \cite{akaike_likelihood_1981},
\begin{equation}
    K := \frac{\prod_i P(E_i | A)}{\prod_i P(E_i | B)} = \frac{\expect{\prod_i P(E_i | x, A)}{x}}{\expect{\prod_i P(E_i | x, B)}{x}}.
\end{equation}
If $K > 1$, then model $A$ is to be preferred as an explanation
of the evidence seen thus far. In particular, the expectation over model parameters
penalizes overfitting, such that a model preferred by $K$ must justify the dimensionality
of $x$. This is made concrete by noting that $K$ is well-approximated by the Bayesian information
criterion when the prior is a multivariate normal distributon \cite{akaike_likelihood_1980,raftery1999bayes}.

Using \CRej~to perform model selection, then, consists of accumulating
subsequent values of $N_a$ in a log-likelihood register $\ell$,
\begin{equation}
    \ell^{(k + 1)} = \ell^{(k)} + \ln\left[(N_a^{(k + 1)} + \beta) / (m + 2 \beta)\right],
\end{equation}
where superscripts are used to indicate the number of Bayes updates performed,
and  $\beta$ is a \emph{hedging parameter} used to prevent
divergences that occur when $N_a = 0$. Since this accumulation procedure
estimates the total likelihood from a two-outcome event
(acceptance/rejection of a sample), the value of $\beta = 1 / 2$ deals with the zero-likelihood case~\cite{ferrie_estimating_2012}.
Letting $\ell_A$ and $\ell_B$ be the hedged log-likelihood registers for models $A$ and $B$,
respectively. Then, the estimator $\hat{K} = e^{\ell_B} / e^{\ell_A}$ resulting from this hedging procedure
is thus an asymtotically-unbiased estimator for $K$ that has well-defined
confidence intervals \cite{cho_approximate_2013}.
The $\ln(m + 2\beta)$ term can be factored out in cases where $m$ is constant across models and evidence.

Model selection of this form has been used, for instance, to decide if
a diffusive model is appropriate for predicting future evidence \cite{granade_characterization_2015}.
Given the aggressiveness of the \CRej~resampling step, streaming model selection
will be especially important in assessing whether a diffusive inference model has
``lost'' the true value \cite{wiebe_efficient_2015}.

\section{Error Analysis}
\label{sec:error-analysis}

Since our algorithms are only approximate, an important remaining issue is that of error propagation in the estimates of the posterior mean
and covariance matrix.  We provide bounds on how these errors can spread and provide asymptotic criteria for stability below.  For notational convenience,
we take $\langle \cdot\!~,\cdot \rangle$ to be the inner product between two distributions and $\|\cdot\|$ to be the induced $2$--norm.

\begin{lemma}
    \label{lem:errprop}

    Let $P(x)$ be the prior distribution and $\tilde{P}(x)$ be an approximation to the prior such that $\tilde{P}(x) = P(x) -\Delta(x)$ and let $P(x|E)$ and $\tilde{P}(x|E)$ be the posterior distributions after event $E$ is observed for $x\in V\subset \mathbb{R}^N$ where $V$ is compact and $\|x\|\le \|x_{\rm max}\|$ for all $x\in V$.  If $|\langle P(E|x),\Delta(x) \rangle|/P(E) \le 1/2$ then the error in the posterior mean then satisfies
    $$
    E_1 \le 4 \frac{\langle P(E|x), |\Delta(x)|\rangle}{P(E)}\|x_{\max}\|,
    $$
    and similarly the error in the expectation of $xx^T$ is
    $$
    E_2 \le 4 \frac{\langle P(E|x), |\Delta(x)|\rangle}{P(E)}\|x_{\max}\|^2,
    $$
where $E_1:=\left\|\int_V  (P(x|E) -\tilde{P}(x|E)) x \mathrm{d}^N x \right\|$ and $E_2 := \left\|\int_V  (P(E|x) -\tilde{P}(E|x)) xx^T \mathrm{d}^N x \right\|$
\end{lemma}

\lem{errprop} shows that the error in the posterior mean using an approximate prior is small given that the inner product of the likelihood function with the errors is small relative to $P(E)$. 

\begin{theorem}\label{thm:meanCov}
If the assumptions of~\lem{errprop} are met and the rejection sampling algorithm uses $m$ samples from the approximate posterior distribution to infer the posterior mean  and $x_j\sim \tilde{P}(x|E)$ then the error in the posterior mean scales as
$$
 O\left(\left[\frac{{N}}{\sqrt{m}} +\frac{\langle P(E|x), |\Delta(x)|\rangle}{P(E)}\right]\|x_{\max}\|\right).
$$
and the error in the estimate of $\Sigma$ is

$$
 O\left(\left[\frac{{N} }{\sqrt{m}} +\frac{\langle P(E|x), |\Delta(x)|\rangle}{P(E)}\right]\|x_{\max}\|^2\right).
$$
\end{theorem}

This theorem addresses the question of when we can reasonably expect the update process discussed in~\alg{crej} to be stable.  By stable, we mean that small initial errors do not exponentially magnify throughout the update process.  \thm{meanCov} shows that small errors in the prior distribution do not propagate into large errors in the estimates of the mean and posterior matrix given that $P(E)= \langle P(E|x),P(x)\rangle$ is sufficiently large.  In particular, \thm{meanCov} and an application of the Cauchy--Schwarz inequality show that such errors are small if $\|x_{\max}\|\le 1$, $m\in \Omega(N^2)$ and 
$$
\langle\Delta(x),\Delta(x)\rangle \ll \frac{P^2(E)}{{\langle P(E|x),P(E|x)\rangle}}.
$$
However, this does not fully address the question of stability because it does not consider the errors that are incurred from the resampling step.

We can assess the effect of these errors by assuming that, in the domain of interest, the updated model after an experiment satisfies a Lipschitz condition
\begin{equation}
\max_x|P_{\mu,\Sigma}(x) - P_{\mu' ,\Sigma'}(x)| \le L(\|\mu- \mu'\| +\|\sqrt{\Sigma}- \sqrt{\Sigma'}\|),
\end{equation}
for some $L\in \mathbb{R}$.  This implies that error in the approximation to the posterior distribution, $\Delta'(x)$ obeys
\begin{equation}
\max_x |\Delta'(x)| \in O\left( \frac{L\int_V P(E|x) \mathrm{d}^Nx \max_x |\Delta(x)|}{P(E)}\right)
\end{equation}
Stability is therefore expected if $\|x_{\max}\|\le 1$, $m\in \Omega(N^2)$ and
\begin{equation}
P(E) \gg {L\int_V P(E|x) \mathrm{d}^Nx }.
\end{equation}
Thus we expect stability if (a) low likelihood events are rare, and (b) the Lipschitz constant for the model is small.  In practice, both of these potential failures can couple together to lead to rapid growth of errors.  It is quite common, for example, for errors in the procedure to lead to unrealistically low estimates of the distribution which causes the Lipschitz constant to become large.  This in turn could coincide with an unexpected outcome to destabilize the learning algorithm. 
We overcome this problem by invoking random restarts, as outlined in the next section,
though other strategies exist~\cite{wiebe_efficient_2015}.

\section{Numerical Experiments}
\label{sec:numerical-experiments}

In this section, we demonstrate rejection filtering both in the context of
learning simple functions in a memory-restricted and time-dependent fashion,
and in the context of learning more involved models such as handwriting
recognition. In both cases, we see that rejection filtering provides significant
advantages.

\subsection{Multimodal Frequency Estimation}

In the main body, we demonstrate the effectiveness of rejection filtering using as an
example strongly multimodal and periodic likelihood functions, such as arise
in frequency estimation problems drawn from the study of quantum mechanical
systems \cite{wiebe_efficient_2015,ferrie_how_2013}.
These likelihood functions serve as useful test cases for
Bayesian inference algorithms more generally, as the multimodality of these
likelihoods forces a tradeoff between informative experiments and
multimodality in the posteriors. Thus, rejection filtering succeeds in these cases only
if our method correctly models intermediate distributions so that appropriate
experiments can be designed.

Concretely, we consider an inference problem with a single observed variable $E\in\{0, 1\}$, decision variables $x_-$ and $t$ and with a single hidden variable $x$ that we allow
to vary with time.
Our aim is to infer the current value of $x$ given values of $E$ obtained for different values of $(x_-,t)$.
The likelihood of measuring $E=1$ for these experiments is
\begin{equation}
  \label{eq:inv-model}
 \Pr(1 | x; t, x_-,k) = \cos^2((x(k) - x_-) t / 2),
\end{equation}
where $k$ is the index of the current update.  The true model $x(k)$ is taken to
follow a random walk with $x(0)\sim \operatorname{Uniform}(0,\pi/2)$ and the distribution
of $x(k+1)$ given $x(k)$ is
\begin{equation}
    x(k+1)=x(k) + \NN(0,(\pi/120)^2).
\end{equation}
The goal in such cases is to identify such drifts and perform active feedback to calibrate
against the drift.

We design experiments $(x_-, t)$ using a heuristic that picks $x_-$ to be a random
vector sampled from the prior and $t=1/\sqrt{{\rm Tr}~ \Sigma}$ \cite{wiebe_efficient_2015}.
We use this heuristic because it is known
to saturate the Bayesian Cramer--Rao bound for this problem~\cite{dauwels_computing_2005,wiebe_hamiltonian_2014} and is
much faster than adaptively choosing $(x_-, t)$ to minimize the Bayes risk, which
we take to be the expected quadratic loss after performing an experiment given
the current prior.

The performance of \CRej~applied to this case is shown in~\fig{crej-diffusion}.
In particular, the median error incurred by \CRej~achieves the optimal
achievable error $(\pi / 120)^2$ with as few as $m = 100$ sampling attempts.
Thus, our \CRej~algorithm continues to be useful in the case of time-dependent
and other state-space models.  Although this demonstration is quite simple, it
is important to emphasize the minuscule memory requirements for this task
mean that this tracking problem can be solved using a memory limited processor in
a different device that is physically close to the system in question.
Close proximity is necessary in applications, such as in control problems or head tracking in virtual reality~\cite{lavalle_sensor_2013,yao2014oculus}, to make the latency
low relative to the dynamical timescale of the object or system that is being tracked.

\subsection{Handwriting Recognition}

A more common task is handwriting recognition.  Our goal in this task is to use
Bayesian inference to classify an unknown digit taken from the MNIST repository~\cite{lecun1998mnist} into
one of two classes.  We consider two cases: the restricted case of classifying only $1$ digits and $0$ digits (zero vs one) and classifying digits as either even or odd (even vs odd).

We cast the problem in the language of Bayesian inference by assuming the likelihood
function
\begin{equation}
P(E|x;i,\sigma)\propto e^{-(x_i - E)^2/2\sigma^2}\label{eq:gausseq},
\end{equation}
which predicts the probability that a given pixel $i$ takes the value $E$,
given that the training image $x$ is the true model that it drew from.

We pick this likelihood function because if we imagine measuring every pixel in the image then the
posterior probability distribution, given a uniform initial distribution, will typically be sharply peaked around
the training vector that is closest to the observed vector of pixel intensities.
Indeed, taking the product over all pixels in an image produces the radial basis function
familiar to kernel learning \cite{scholkopf_learning_2001}.

Unlike the previous experiment, we do not implement this in a memory--restricted
setting because of the size of the MNIST training set.
Our goal instead is to view the image classification problem through the lens of active feature learning.  
We assume that the user has access to the entire training set, but wishes to classify a digit by reading as few of the training vectors features (pixels) to the training
image as possible.  Although such a lookup is not typically expensive for MNIST, in more general tasks, features may be very expensive to compute, and our experiment shows that our approach enables identification of important features.  This reduces the computational requirements of the classification task, and also serves to identify salient features for future 
classification problems.

Such advantages are especially relevant to search wherein features, such as term occurences and phrase matches across terms, can be expensive to compute on the fly.  It also can
occur in experimental sciences where each data point may take minutes or hours to either measure
or process.  In these cases it is vital to minimize the number of queries made to the training
data.  We will show that our Bayesian inference approach to this problem allows this to be solved using
far fewer queries than other methods, such as kNN, would require.  We also show that our method can
be used to \emph{extract} the relevant important features (i.e. pixels) from the data set.  This is examined further in
\app{features}.

We perform these experiments using an adaptive guess heuristic, similar to that employed in the frequency estimation example.
The heuristic chooses the pixel label, $i$, that has the largest variance of intensity
over the $m$ training vectors that compose the particle cloud.  We then pick $\sigma$ to be the standard
deviation of the intensity of that pixel.  As learning progresses the diversity in the set of particles
considered shrinks, which causes the variance to decrease.  Allowing $\sigma$ to shrink proportionally accelerates the inference process by amplifying the effect of small differences in $P(E|x)$.  

\begin{figure}
\centering
\includegraphics[width=0.6\columnwidth]{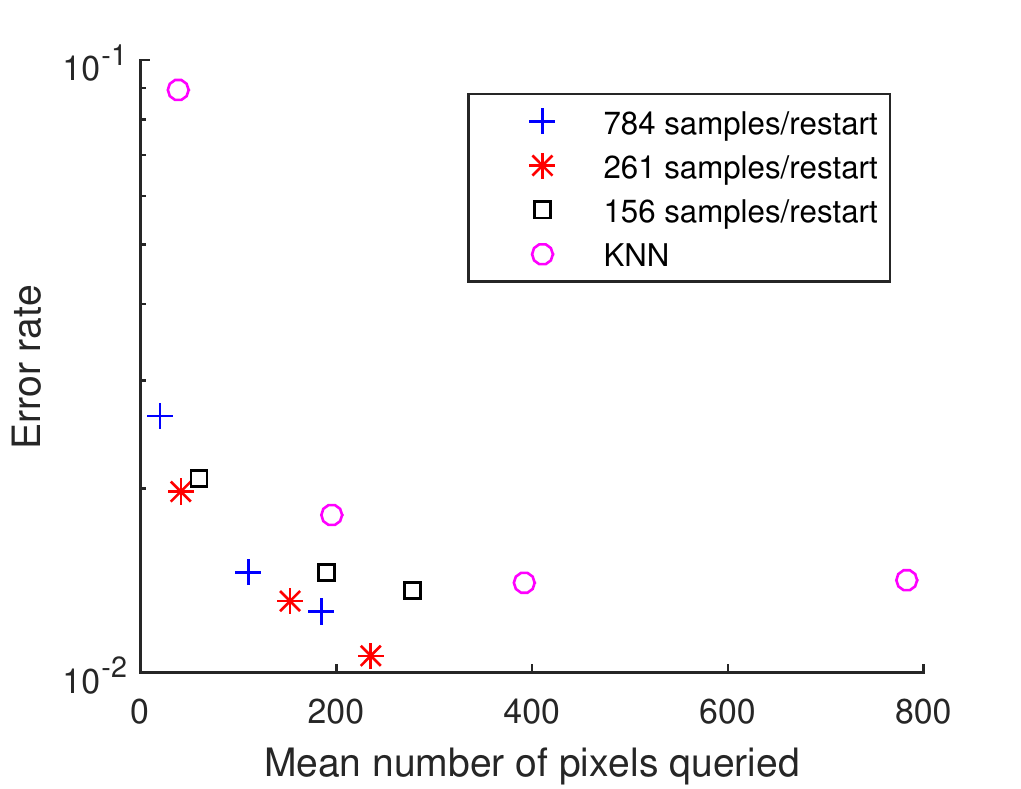}
\caption{Classification errors for odd/even MNIST digits for rejection filtering with 784 maximum experiments distributed over 1,3,5 restarts and stopping condition $\mathcal{P}=0.1,0.01,0.001$.}\label{fig:errorplot}
\end{figure}

We repeat this process until the sample probability distribution has converged to a distribution that
assigns probability at most $\mathcal{P}$ to one of the two classes in the problem.  This process is restarted a total of
$1,3$ or $5$ times subject to the constraint that at most $784$ queries are made (divided equally over each of the restarts).
The label assigned to the test vector is then the most frequently appearing label out of these tests.  This ensures that our method
incurs at most the same cost as na\i ve kNN, but is pessimistic as we do not allow the algorithm to store the results of prior queries which could reduce the 
complexity of the inference.

 The resampling step as described in~\alg{crej} is unnatural here because there are only two classes rather than a continuum.  Instead, we use a method similar to the bootstrap filter.  Specifically, when resampling we draw a number of particles from both
classes proportional to the sample frequency in the posterior distribution.  For each of these classes we then replicate the surviving particles until $95\%$ of the total population is replenished
by copies from the posterior cloud of samples.  The remaining $5\%$ are drawn uniformly from the training vectors in the corresponding class.

\fig{HM} illustrates the differences between maximum variance experiment design and the adaptive method we use in rejection filtering.  These differences
are perhaps most clearly seen in the problem of classifying one vs zero digits from MNIST.  Our adaptive approach most frequently queries the middle pixel, which is the 
highest variance feature over the training set.  This is unsurprising, but the interesting fact is that the second most frequently queried pixel is one that has relatively low variance
over the training set.  In contrast, many of the high-variance pixels near the maximum variance pixel are never queried despite the fact that they have large variance over the set.
This illustrates that they carry redundant information and can be removed from the training set.  Thus this adaptive learning algorithm can also be used to provide a type of model compression, similar to feature selection
or PCA.

The case for odd vs even is more complicated.  This is because the handwritten examples have much less structure in them and so it is perhaps unsurprising that less dramatic compression is possible when examining that class. However, the data still reveals qualitatively that even here there is a disconnect between the variance of the features over the training set and their importance in classification.
An interesting open question is how well rejection filtering
can work to select features for other classification methods.

We will now compare rejection filtering to kNN classification.  kNN is a standard approach for digit classification that performs well but it can be prohibitively slow for
large training sets \cite{yu2010high} (indexing strategies can be used to combat this problem~\cite{yu2001indexing}).  In order to make the comparison as fair as possible between the two, we compare kNN to rejection filtering by truncating the 
MNIST examples by removing pixels that have low variance over the training set.  This removes some of the advantage our method has by culling pixels near
the boundary of the image that contain very little signal  (see~\fig{HM}) and yet substantially contribute to the cost of kNN in an active learning setting.

Feature extraction~\cite{zhang2006svm,weinberger2008fast,min2009deep} or the use of 
deformation models~\cite{keysers2007deformation} can also be used to improve the performance of kNN.
We ignore them here for simplicity because they will also likely improve rejection filtering.

\fig{errorplot} shows that in certain parameter regimes, approximate Bayesian inference via rejection sampling can not only achieve higher classification accuracy on average for a smaller
number of queries to the test vector, but also can achieve $25\%$ less error even if these constraints are removed.  This result is somewhat surprising given that we chose our likelihood function to correspond to nearest neighbor classification if $\sigma$ is held constant.  However, we do not hold $\sigma$ constant in our inference but rather choose it adaptively as the experiment proceeds.  This changes the weight of the evidence provided by each pixel query and allows our algorithm to outperform kNN classification despite the apparent similarities between them.

\section{Proofs of Theorems}
\label{app:proofs}

In this Appendix, we present proofs for the theorems presented in the main
body.

\begin{proofof}{\thm{crej}}
There are two parts to our claim in the theorem.  The first is that the rejection sampling algorithm is efficient given the theorem's assumptions
and the second is that it only requires $O({\rm dim}(x)^2 \log(1/\epsilon))$ memory to approximate the appropriate low--order moments of
 the posterior distribution.

For each of the $m$ steps in the algorithm the most costly operations are 
\begin{enumerate}
\item Sampling from $P$.
\item Sampling from the uniform distribution.
\item The calculation of $xx^T$.
\end{enumerate}
The first two of these are efficient by the assumptions of the theorem.  Although it may be tempting to claim that efficient algorithms are known
for sampling from the uniform distribution, the existence of such deterministic algorithms is unknown since it is not known whether the complexity
classes $\BPP$ and $\P$ coincide.  The remaining operation can be computed using $O({\rm dim}(x)^3)$ arithmetic operations, each of which can
be performed (to within bounded accuracy) efficiently on a Turing machine.  Therefore the cost of the inner loop is $O(m{\rm dim}(x)^3)$ which is efficient
if $m$ is taken to be a constant.

The remaining operations require at most $O({\rm dim}(x)^3)$ arithmetic operations and thus do not dominate the cost of the algorithm.  The main question remaining
is how large $m$ needs to be and how many bits of precision are required for the arithmetic.  Both the error in the mean and the elements of the covariance matrix scale as $O(1/\sqrt{N_a})$ where $N_a$ is the number of accepted samples that pass through the rejection filter.  Thus if both are to be computed within error $\epsilon$ then $N_a \in O(1/\epsilon^2)$.  However, in order to get a sample accepted we see from the Markov inequality and the definition of the exponential distribution that $m$ must scale like $m\in O(1/P_{\rm success} \epsilon^2)$.  We then see from~\cor{badalgorithm} that $P_{\rm success} \in \Omega(\min_x P(E|x)/\kappa_E)$, which we assume is at most polynomially small.  Ergo the sampling process is efficient given these assumptions and the fact that $\epsilon$ is taken to be a constant for the purposes of defining efficiency.

The dominant requirements for memory arise from the need to store $\Sigma$, $\mu\mu^T$ and $xx^T$.  There are at most $O({\rm dim}(x)^2)$ elements in those matrices and so if each is to be stored within error $\epsilon$ then at least $O({\rm dim}(x)^2\log(1/\epsilon))$ bits are required.  Note that the incremental formulas used in the algorithm are not very numerically stable and need $2N$-bit registers to provide an $N$-bit answer.  This necessitates doubling the bits of precision, but does not change the asymptotic scaling of the algorithm.  Similarly, the $m\in O(1/\epsilon^2)$ repetitions of the algorithm also does not change the asymptotic scaling of the memory because $\log(1/\epsilon^3) \in O(\log(1/\epsilon))$.

What does change the scaling is the truncation error incurred in the matrix multiplication.  The computation of a row or column of $xx^T$, for example, involves ${\rm dim}(x)$ multiplications and additions.  Thus if each such calculation were computed to to within error $\epsilon$, the total error is at most by the triangle inequality ${\rm dim}(x) \epsilon$.  Therefore in order to ensure a total error of $\epsilon$ in each component of the matrix we need to perform the arithmetic using $O(\log({\rm dim}(x)/\epsilon))$ bits of precision.  

The incremental statistics, for both quantities, involve summing over all $m$ particles used in the rejection filtering algorithm.  If we assume that fixed point arithmetic is used to represent the numbers then we run the risk of overflowing the register unless its length grows with $m$.  
The result then follows.
\end{proofof}

\begin{proofof}{\lem{errprop}}
Using the definition of $\tilde{P}(x)$ and Bayes' rule it is easy to see that the error in the posterior mean is
\begin{align}
\Biggr| \int_V \frac{P(E|x)P(x)x}{\langle P(E|x),P(x) \rangle}\left( 1 - \frac{1}{1+\frac{\langle P(E|x),\Delta(x)\rangle }{\langle P(E|x),P(x) \rangle}}\right) - \frac{P(E|x) \Delta(x)x}{P(E)}\left(\frac{1}{1+\frac{\langle P(E|x),\Delta(x)\rangle }{\langle P(E|x),P(x) \rangle}} \right)\mathrm{d}^Nx \Biggr|.\label{eq:intbd}
\end{align}
Using the fact that $|1-1/(1-y)| \le 2|y|$ for all $y\in [-1/2,1/2]$ it follows from the assumptions of the theorem and the triangle inequality that~\eq{intbd} is bounded above by
\begin{align}
 \int_V \frac{2P(E|x)P(x)\|x\| |\langle P(E|x),\Delta(x) \rangle|}{P(E)^2}\mathrm{d}^Nx+ \int_V\frac{2P(E|x) |\Delta(x)|\|x\|}{P(E)}\mathrm{d}^Nx.\label{eq:intbd2}
\end{align}
Now using the fact that $\|x\|\le \|x_{\max}\|$ and the definition of the inner product, we find that~\eq{intbd2} is bounded above by
\begin{equation}
\frac{2 (|\langle P(E|x),\Delta(x) \rangle| + \langle P(E|x), |\Delta(x)| \rangle))\|x_{\max}\|}{P(E)}.
\end{equation}
The  result then follows from a final application of the triangle inequality.

The analogous proof for the error in the posterior expectation of $xx^T$ follows using the exact same argument after replacing the Euclidean norm with the induced $2$--norm for matrices.  Since both norms satisfy the triangle inequality, the proof follows using exactly the same steps after observing that $\|xx^T\|\le \|x_{\max}\|^2$ for all $x\in V$.
\end{proofof}

\begin{proofof}{\thm{meanCov}}
\lem{errprop} provides an upper bound on the error in the mean of the posterior distribution that propagates from errors in the components of our prior distribution.  We then have that if we sample from this distribution then the sample standard deviation of each of the $N$ components of $x$ is $O(\sigma_{\max}/\sqrt{m})$.  Thus from the triangle inequality the sample error in the sum is at most 
\begin{equation}
O\left(\frac{{N}\sigma_{\max}}{\sqrt{m}}\right)\in O\left( \frac{N\|x_{\max}\|}{\sqrt{m}}\right).
\end{equation}  
The triangle inequality shows that the sampling error and the error that propagates from having an incorrect prior are at most additive.  Consequently the total error in the mean is at most the the sum of this error and that of \lem{errprop}.  Thus the error in the mean is
\begin{equation}
O\left(\left[\frac{N}{\sqrt{m}}+ \frac{\langle P(E|x),|\Delta(x)|\rangle}{P(E)}\right]\|x_{\max}\| \right)\label{eq:mnerr}
\end{equation}

 The calculation for the error in the estimate of the covariance matrix is similar.  First, note that $1/(m-1)= 1/m +O(1/m^2)$ so we can asymptotically ignore $m/(m-1)$.  Let $\mu = \int_V P(x|E) x\mathrm{d}x +\epsilon v$, where $\|v\|\le 1$.  We then have from our error bounds on the estimate of the posterior mean that
\begin{eqnarray}
\| \mu\mu^T &-& \int_V P(x|E) x\mathrm{d}x \int_V P(x|E) x^T\mathrm{d}x \|\nonumber \\&\le& \epsilon \left\|\int_V P(x|E) x\mathrm{d}^Nx v^T\right\|+\epsilon\left\|v\int_V  P(x|E) x^T\mathrm{d}^Nx\right\| + O(\epsilon^2).\nonumber\\
&\in&  O\left(\left[\frac{{N}}{\sqrt{m}} +\frac{\langle P(E|x), |\Delta(x)|\rangle}{P(E)}\right]\|x_{\max}\|^2\right),\label{eq:garbage}
\end{eqnarray}
where we substitute~\eq{mnerr} for $\epsilon$ and use $\int_V P(x|E) x\mathrm{d}^Nx\le \|x_{\rm max}\|$.

Now let us focus on bounding the error in our calculation of $\int_V P(E|x) xx^T \mathrm{d}x$. Using the triangle inequality, the error in the estimate of the expectation value of $x x^T$ is, to within error $O(1/m^{3/2})$, at most
\begin{equation}
\Biggr\|\frac{1}{m} \sum_{j=1}^m x_j x_j^T - \int_V \tilde{P}(x|E) x x^T\mathrm{d}^Nx\Biggr\|+\Biggr\| \int_V \tilde{P}(x|E) x x^T\mathrm{d}^Nx-\int_V {P}(x|E) x x^T\mathrm{d}^Nx\Biggr\|.\label{eq:xxT}
\end{equation}
The first term in \eq{xxT} can be bounded by bounding the sample error in each of the components of the matrix.  For any component $[xx^T]_{k,\ell}$ the Monte--Carlo error in its estimate is
\begin{equation}
O\left(\frac{\sigma({[x]_k[x]_\ell})}{\sqrt{m}}\right)\in O\left(\frac{\|x_{\max}\|^2}{\sqrt{m}}\right).
\end{equation}
The $2$--Norm of an $N\times N$ matrix is at most $N$ times its max--norm, which means that
\begin{equation}
\Biggr\|\frac{1}{m} \sum_{j=1}^m x_j x_j^T - \int_V \tilde{P}(x|E) x x^T\mathrm{d}^Nx\Biggr\|\in O\left(\frac{N\|x_{\max}\|^2}{\sqrt{m}}\right).\label{eq:xxTMC}
\end{equation}
The theorem then follows from inserting~\eq{xxTMC} into~\eq{xxT} and applying~\lem{errprop}, and combining the result with~\eq{garbage} to bound the error in the covariance matrix.
\end{proofof}

Note that in the previous theorem that we do not make assumptions that the components of $x$ are iid.  If such assumptions are made then tighter bounds can be proven.

\section{Sensitivity to $\kappa_E$}
\label{app:sensitivity-kappa}

\begin{figure}
    \begin{center}
        \includegraphics[width=0.75\textwidth]{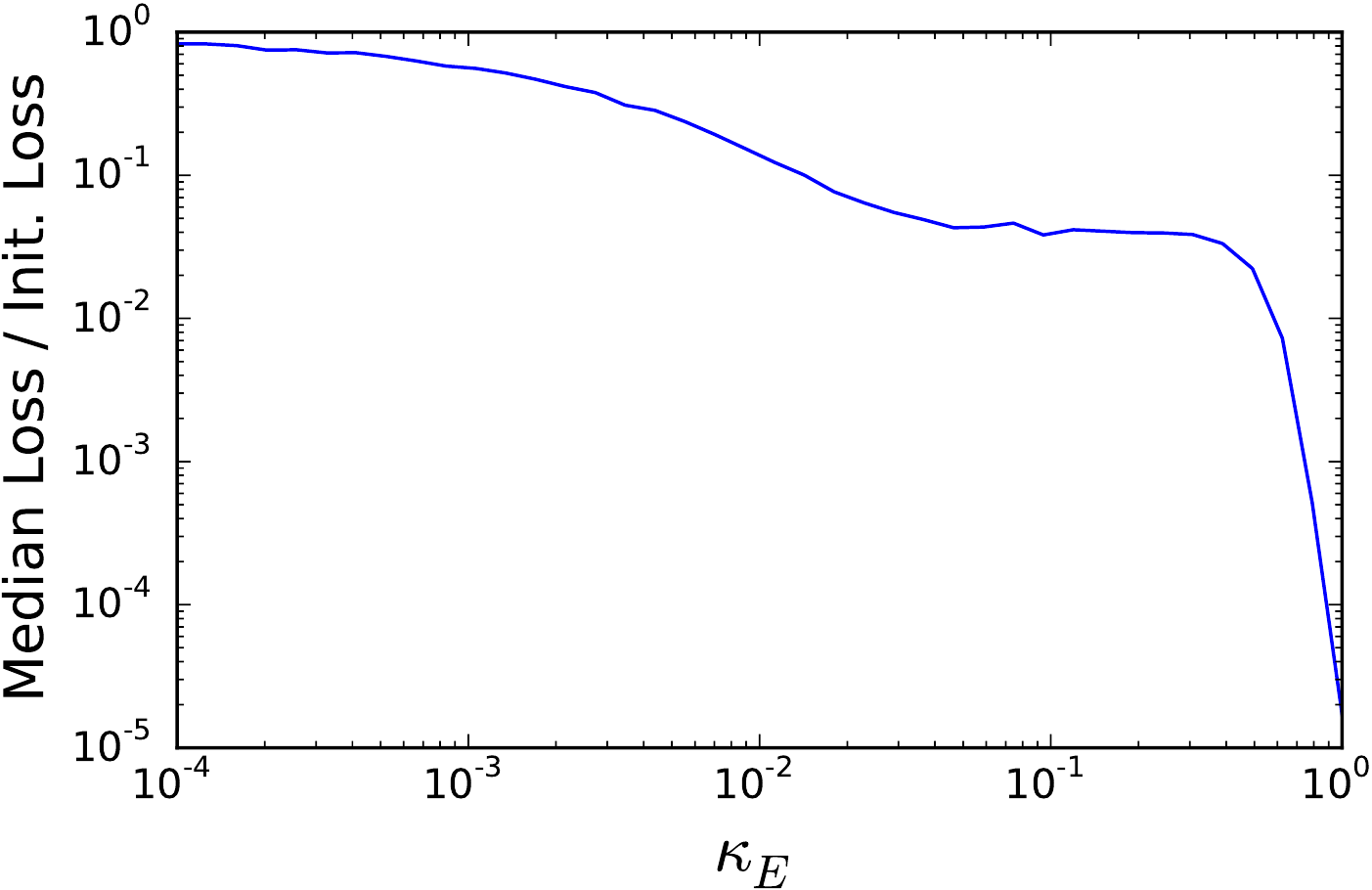}
    \end{center}
    \caption{\label{fig:sensitivity-kappa}
        Sensitivity of inversion model {inv-model} with PGH. The median loss
        over 6,000 trials after 100 measurements is shown as a function of $\kappa_E$.
        For clarity, we normalize the median loss by the initial loss (that is, the median
        loss
        before performing any rejection filtering updates). 
    }
\end{figure}

Theorem 1 in the main body guarantees a bound on the error in a particular Bayes update due to
$\kappa_E$ being chosen smaller than $\max_x P(E|x)$ for evidence $E$.
However, this only holds for a single update, such that it is not immediately clear
how such errors propagate for an entire \emph{set} of experiments. This is especially
important for heuristic online experiment design, such as via the particle guess heuristic,
or for active learning. To address this, here we show the effects of choosing an inappropriate $\kappa_E$
for the frequency estimation model of {inv-model}. 

In particular,
we use the particle guess heuristic to design 100 measurements~\cite{wiebe_efficient_2015}, then update using
and 100 rejection filter attempts per measurement, with a recovery factor of
$r=2\%$. We then report in \fig{sensitivity-kappa} the median loss of this procedure as a function
of $\kappa_E$. Importantly, since the this experimental guess heuristic tends to result in likelihoods with
peaks of order unity, any $\kappa_E < 1$ will generate bad samples with high probability over the $m$ particles.

We see that taking $1>\kappa_E\gtrsim 2/3$ does not prevent the algorithm from learning the true model rapidly, although it certainly does
degrade the quadratic loss after $100$ measurements.  On the other hand, as $\kappa_E$ approaches zero the evidence of learning dissapears.  This is to be expected since if $P(E|x)>1$ for all $x$ then all samples are accepted and the approximate posterior will always equal the prior.  We similarly notice evidence for modest learning for $\kappa_E < 0.5$, and note that a plateau appears in the loss ratio for $\kappa_E \in [0.04,0.4]$ where the loss ratio does not monotonically grow with $1/\kappa_E$.  This illustrates that learning still can take place if a value of $\kappa_E$ is used that
is substantially less than the ideal value.  The ability to tolerate small values of $\kappa_E$ is important in cases where $\kappa_E$ is not known apriori and empirical values, such as those obtained via Monte--Carlo sampling, must be used to make the rejection probability polynomially small.

\section{Feature Selection for Rejection Sampling}
\label{app:features}

In the main body we alluded to the fact that the frequency with which a feature is selected in rejection filtering can be used as an
estimate of the importance of that feature to the classification.  We saw evidence that this application would be reasonable from
noting in MNIST classification problems that some pixels were never, or very infrequently, used and so should be able to be culled from the 
vectors without substantially affecting the classification accuracy.  

\begin{figure}[t!]
    \begin{center}
        \includegraphics[width=0.5\textwidth]{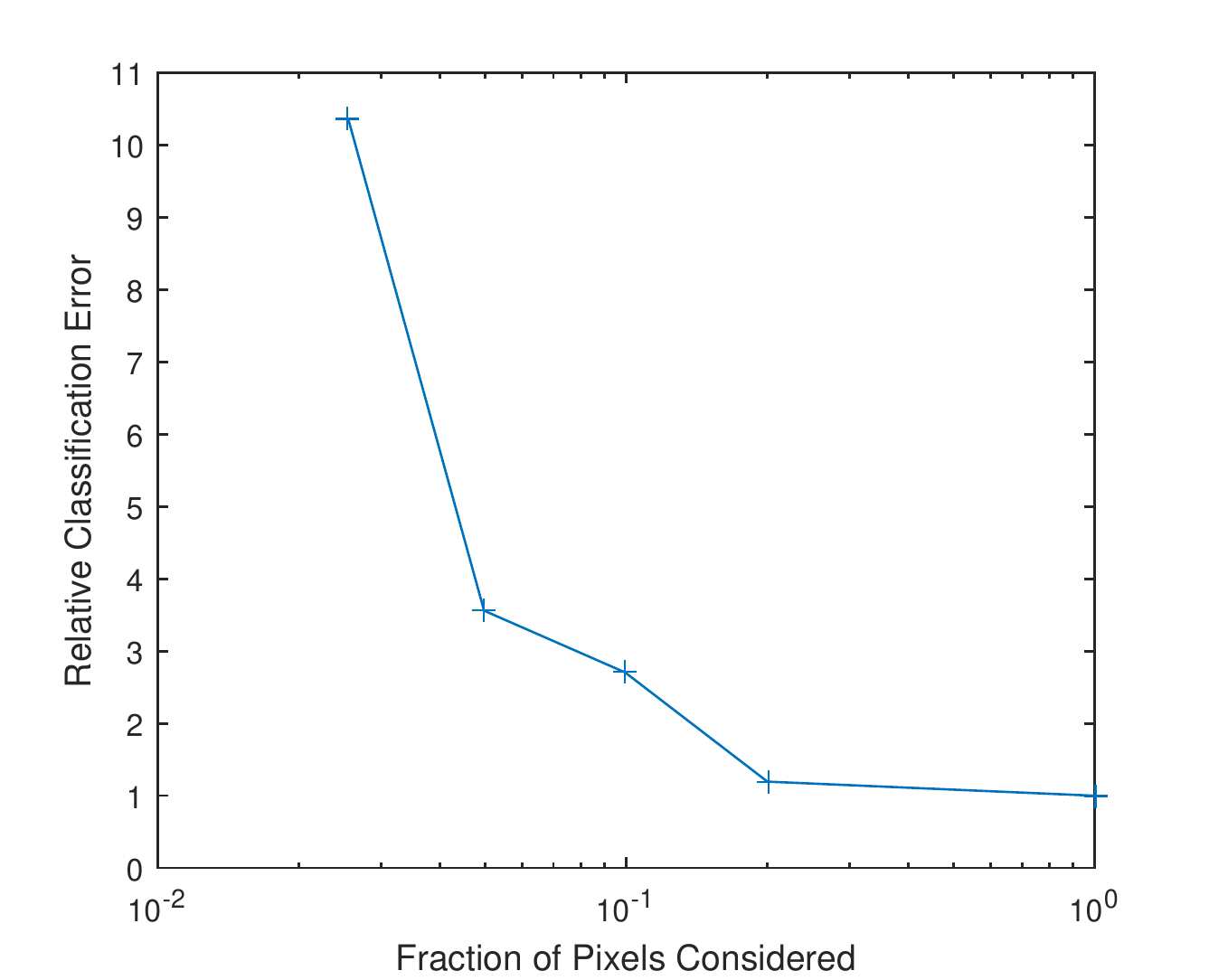}
    \end{center}
    \caption{\label{fig:Features}
Scaling of the classification error as a function of the number of pixels considered in MNIST digit classification.
    }
\end{figure}

Here we investigate this further for the problem of classifying zero vs one digits within MNIST.  We consider the case where $3$ restarts, each
consisting of at most $261$ samples, are performed with $\mathcal{P}=0.001$.  Our protocol begins by computing the frequency that pixels are queried over
the full case where all $784$ features in the MNIST images are used and then computing a histogram of how many samples were drawn for each pixel.  Then we divide the pixels
into sets based the relative frequency that they were used.  For example, we find all pixels whose frequency is less than the $95^{\rm th}$ percentile and remove the offending pixels from the entire MNIST set.  We then perform the classification on using only these pixels and compare to the original classification accuracy.  The number of features as a function of the percentile considered for zero vs one is given below.

\begin{table}[h!]
\centering
\begin{tabular}{|c|c|}
\hline
Percentile & Number of Features\\
\hline
0 & 784\\
35 &784\\
36 &503\\
50 &392\\
75&196\\
80&157\\
90&78\\
95&39\\
97.5&20\\
\hline
\end{tabular}
\end{table}

This strongly suggests that truncating at least $35\%$ of the data will not impact the classification whatsoever.  This does not, however, mean that the pixels that become significant at the $36^{\rm th}$ percentile will substantially affect classification accuracy since an acceptable classification might be reached using other, more significant, features in the set.  

We examine the feature extraction ability of rejection filtering by looking at the mean classification accuracy over $100$ different random shufflings of the vectors in the two classes for MNIST digits, using a $10$ to $1$ train--test split.  The data is shown in~\fig{Features}.  We see that if we use the frequency data provided by rejection sampling to remove all but the most significant $20\%$ of the training data then the classification accuracy is not significantly reduced.  After this point the impact of ignoring features begins to grow roughly linearly with the fraction of points ignored.  This continues until only the training examples are trimmed down to $95\%$ of their original length (39 features) before the classification accuracy becomes strongly impacted by removing features from consideration.  Nonetheless, for all the points considered the classification accuracy remains high (greater than $99\%$).

This shows that we can use this approach to choose salient features from an ensemble and by following this criteria we can cull unimportant features from a large set based on the frequencies that the data was needed in training rejection filtering experiments.  While this process shows advantages in this simple case, more involved experiments are needed to compare its performance to existing feature extraction methods such as PCA and to determine whether adaptive schemes that iteratively find the most important features as ever increasing portions of the original vectors are culled will be needed to achieve good performance in complex data sets.


\section{Batched Updating}
\label{app:batched-updates}

Although we focused in the main body on memory restricted applications, it is also possible to exploit the fact that the
rejection sampling procedure is inherently parallelizable.
This comes at the price of increasing the overall
memory usage. In \alg{batchcrej}, we describe a batched form of our algorithm, assuming a model in which samples are sent by a server to individual processing nodes and the accepted samples are then returned to the server.

\begin{algorithm}
    \caption{Batched update for \CRej}
    \label{alg:batchcrej}
    \begin{algorithmic}
        \Require Prior distribution $\pi:\mathbb{R}^D \mapsto [0,1]$, array of evidence $\vec{E}$, number of attempts $m$, a constant $0<\kappa_E\le 1$, a recovery factor $r \ge 0$, the prior mean $\mu$ and the covariance matrix $\Sigma$.
        \Ensure  The mean and coviariance matrix of updated distribution $\mu$,$\Sigma$ and $N_a$ which is the number of samples accepted.
        \Function{BatchUpdate}{$\vec{E}$, $m$, $\kappa_E$, $\mu$, $\Sigma$, $r$, $N_{\rm batch}$}
  \State{$(M,S,N_a) \gets 0$}
          \For{{\bf each} $i \in 1 \to N_{\rm batch}$}
  \State Pass $\vec{E},m,\kappa_E, \mu,\Sigma,r$ to processing node $i$.
  \State Set local variables $(\mu^{i}, \Sigma^{i},N_a^{i})\gets \Call{RFUpdate}{\vec{E},m,\kappa_E,\mu,\Sigma,r}$.
  \If{$N_a^i >0$}
  \State $M^i \gets \mu^i N_a^i$
  \State $S^i \gets (N_a^i-1)\Sigma +N_a^i\mu\mu^T$ 
  \State Pass $N_a^i$, $M^i$ and $S^i$ back to the server.       
  \Else \State Pass $(0,0,0)$ back to the server.
  \EndIf
  \EndFor
    \If{$\sum_i N_a^i > 0$}
     \State $\mu\gets \sum_i M^{i}/\sum_i N_a^i $
     \State $\Sigma \gets \frac{1}{\sum_i N_a^i -1}\left(\sum_i S^i - \sum_i N_a^i \mu\mu^T \right)$
  \State\Return $(\mu,\Sigma,N_a)$
   \Else
  \State\Return $(\mu, (1+r)\Sigma),N_a)$

   \EndIf
          
        \EndFunction
    \end{algorithmic}
\end{algorithm}

\bibliographystyle{unsrt}
\bibliography{qsmc}
\end{document}